\documentclass[oneside,english]{article}
\usepackage[LGR,T1]{fontenc}
\usepackage[utf8]{inputenc}
\usepackage[a4paper]{geometry}
\usepackage{color}
\usepackage[ngerman]{babel}
\usepackage{array}
\usepackage{varioref}
\usepackage{float}
\usepackage{units}
\usepackage{textcomp}
\usepackage{amsthm}
\usepackage{amsmath}
\usepackage{amssymb}
\usepackage{amsfonts}
\usepackage{graphicx}
\usepackage{setspace}
\usepackage{nomencl}
\usepackage{umlaute}
\usepackage{enumitem}

\usepackage[notocbib]{apacite}

\usepackage[bottom,hang]{footmisc}
\usepackage{geometry}
\usepackage{ifthen}
\usepackage{nomencl}

\makenomenclature
\renewcommand{\nomgroup}[1]{%
\renewcommand{\makelabel}[1]{##1}
\item[~]

\ifthenelse{\equal{#1}{D}}{%
\item[\textbf{Distributions}]}{%

\ifthenelse{\equal{#1}{S}}{%
\item[\textbf{Sets and spaces}]}{%

\ifthenelse{\equal{#1}{F}}{%
\item[\textbf{Functions}]}{%

\ifthenelse{\equal{#1}{M}}{%
\item[\textbf{Miscellaneous}]}{%

\ifthenelse{\equal{#1}{O}}{%
\item[\textbf{Observations}]}{%

\ifthenelse{\equal{#1}{L}}{%
\item[\textbf{Losses and risks}]}{%

\ifthenelse{\equal{#1}{C}}{%
\item[\textbf{Cardinalities}]}{%

}}}}}}}%
\item[~]
\let\makelabel\nomlabel
}

\newcommand{\R}{\mathbb{R}}

\newcommand{\N}{\mathbb{N}}

\newcommand{\BorelXY}{{\mathfrak{B}_{\X\times\Y}}}
\newcommand{\BorelXnbY}{{\mathfrak{B}_{\Xb\times\Y}}}

\newcommand{\BorelX}{{\mathfrak{B}_{\X}}}

\newcommand{\BorelY}{{\mathfrak{B}_{\Y}}}

\newcommand{\Pb}{{{P}_{b}}}

\newcommand{\PbX}{{ {P}_b^X }}
\newcommand{\E}{{\mathbb{E}}}

\newcommand{\Y}{{\mathcal{Y}}}
\newcommand{\X}{{\mathcal{X}}}
\newcommand{\Xb}{{\mathcal{X}_{b}}}

\newcommand{\XI}{{\X_{I}}}

\newcommand{\frakDn}{{{D}_n}}                       % empirische Verteilung als Realisierung global
\newcommand{\frakDnb}{{{D}_{n,b}}}              % empirische Verteilung als Realisierung lokal

\newcommand{\calDn}{{\mathcal{D}_n}}                        % Datensatzrealisierung
\newcommand{\calDnb}{{\mathcal{D}_{n,b}}}             % Teildatensatzrealisierung

\newcommand{\calDN}{{\mathcal{D}_N}}                        % alle Daten, inkl. Trainingsdatensatz fuer Regionalisierung

\newcommand{\Ind}{{\mathbf{1}}}

\newcommand{\bsumme}{{\sum\limits_{b=1}^{B}}}

\newcommand{\insumme}{{\sum\limits_{i=1}^n}}
\newcommand{\jnsumme}{{\sum\limits_{i=j}^n}}

\newcommand{\Isumme}{{ \sum\limits_{I \subset \btoB   }}}

\newcommand{\nb}{{n_b}}

\newcommand{\lambdab}{{\lambda_b}}
\newcommand{\lambdanb}{{\lambda_{(n_b,b)}}}

\newcommand{\fLPlambdacomp}{{f_{\Ls,P,\lambda}^{comp}}} 
\newcommand{\fLDnlambdacomp}{{f_{\Ls,\frakDn,\lambda}^{comp}}} 
\newcommand{\fLPschlangelambdacomp}{{f_{\Ls,\tilde P,\lambda}^{comp}}}

\newcommand{\fLPlambdancomp}{{f_{\Ls,P,\lambda_n}^{comp}}} 
\newcommand{\fLDnlambdancomp}{{f_{\Ls,\frakDn,\lambda_n}^{comp}}}

\newcommand{\Ls}{{ {L^*}   }}
\newcommand{\lokaltheoretisch}{{f_{b,\Ls,\Pb,\lambdab}}} % lokal theoretisch
\newcommand{\lokalempirisch}{{f_{b,\Ls,\frakDnb,\lambdab}}} % lokal empirisch
\newcommand{\lokaltheoretischschlange}{{f_{b,\Ls,\tilde\Pb,\lambdab}}} % lokal schlange
\newcommand{\lokaltheoretischn}{{f_{b,\Ls,\Pb,\lambdanb}}} % lokal theoretisch mit n
\newcommand{\lokalempirischn}{{f_{b,\Ls,\frakDnb,\lambdanb}}} % lokal empirisch mit n
\newcommand{\zusammentheoretischx}{{ {\bsumme w_b(x) \lokaltheoretisch(x) }}}

\newcommand{\zusammenempirischx}{{ {\bsumme w_b(x) \lokalempirisch(x) }}}

\newcommand{\zusammenempirischnx}{{ {\bsumme w_b(x) \lokalempirischn(x)  }}}
\newcommand{\nto}{{\ \xrightarrow[{\tiny n\to\infty}]{}\ }}

\newcommand{\mto}{{\ \xrightarrow[{\tiny m\to\infty}]{}\ }}

\newcommand{\ohne}{{\backslash}}

\newcommand{\btoB}{{\left\{1,\ldots,B\right\}}}

\DeclareMathOperator*{\arginf}{arg\,inf}

\newtheoremstyle{break}% name
	  {10pt}%      Space above, empty = 'usual value'
 	  {10pt}%      Space below
 	  {\itshape}% Body font, Z. B. \itshape
 	  {\parindent}%         Indent amount (empty = no indent, \parindent = para indent)
 	  {\bfseries}% Thm head font, z. B. \bfseries, \itshape
 	  {}%        Punctuation after thm head
 	  {\newline}% Space after thm head: \newline = linebreak
 	  {}%         Thm head spec
\theoremstyle{break}

\addto\captionsngerman{
\renewcommand{\bibname}{References}
\renewcommand{\refname}{References}

}

\newtheorem{stz}{Satz}[section]

\newtheorem{lem}[stz]{Lemma}
\newtheorem{bsp}[stz]{Example}

\newtheorem{bem}[stz]{Remark}

\newtheorem{thm}[stz]{Theorem} %% ACHTUNG geändert
\newtheorem{prp}[stz]{Proposition}

\usepackage{remreset}
\makeatletter
\@removefromreset{footnote}{section}
\makeatother

\title{\textbf{Universal Consistency and Robustness of Localized Support Vector Machines}}
\author{\textbf{Florian Dumpert}\\
\normalsize{Department of Mathematics, University of Bayreuth, Germany}}
\date{}

\begin{document} 

\selectlanguage{english}

\maketitle
 
\begin{abstract}\noindent
The massive amount of available data potentially used to discover patters in machine learning is a challenge for kernel based algorithms with respect to runtime and storage capacities. Local approaches might help to relieve these issues. From a statistical point of view local approaches allow additionally to deal with different structures in the data in different ways. This paper analyses properties of localized kernel based, non-parametric statistical machine learning methods, in particular of support vector machines (SVMs) and methods close to them. We will show there that locally learnt kernel methods are universal consistent. Furthermore, we give an upper bound for the maxbias in order to show statistical robustness of the proposed method. 
\end{abstract}  
  

%\noindent{\bf AMS Subject Classification Numbers.} 60F99, 62G20, 46E22, 68Q32.

%%%%%%%%%%%%%%%%%%%%%%%%
%               Ende  Kopf                                    %
%%%%%%%%%%%%%%%%%%%%%%%%

\section{Introduction}

%\section {General Subsumption}

This paper analyses properties of localized kernel based, non-parametric statistical machine learning methods, in particular of support vector machines (SVMs) and methods close to them. Caused by the enormous research activities there is abundance of general introductions to this field of computer science and statistics. Beside many publications in international journals there are summarizing textbooks like for example \fullciteA{cristianini2000introduction}, \fullciteA{scholkopf2001learning}, \fullciteA{steinwart2008support} or \fullciteA{cucker2007learning} from a mathematical or statistical point of view. Nevertheless, we want to give a short overview over the analyzed topic.

Support vector machines were initially introduced by \fullciteA{boser1992training}  und \fullciteA{cortes1995support}, based on earlier work like the Russian original of \fullciteA{vapnik1979theorie}. The basic ideas are presented in a comprehensive way in \fullciteA{vapnik1995thenature, vapnik2000thenature2} and \fullciteA{vapnik1998statistical}. An early discussion is provided by \fullciteA{bennett2000support}. Although SVMs and related kernel based methods are much more recent then other very well-established statistical techniques like for example ordinary least squares regression or their related generalized linear models for regression and classification, they became pretty popular in many fields of science, see for example \fullciteA{ma2014support}.
%; for example: medical and health research and bioinformatics (e.\,g. \fullciteA{akay2009support}, \fullciteA{polat2007breast}, \fullciteA{markowetz2001support}, \fullciteA{markowetz2003support}, \fullciteA{Guyon2002}) and official statistics (e.\,g. \fullciteA{FeuerhakeDumpert2016}, \fullciteA{DumpertVonEschwegeBeck2016}).
The analysis provided by this paper usually refers to classification or regression problems and therefore to so called supervised learning. Beyond this, support vector machines are a suitable method for unsupervised learning (e.\,g. novelty detection), too.

The paper is organized as follows: Section~\ref{lmusvm} gives an overview on support vector machines, Section~\ref{local} introduces the idea of local approaches. The consistency and robustness results are provided in Section~\ref{results}, the proofs can be found in the appendix. Section~\ref{summary} summarizes the paper.

\section{Prerequisites}

\subsection{Learning Methods and Support Vector Machines}\label{lmusvm}

The aim of support vector machines in our context, i.\,e. in supervised learning, is to discover the influence of a (generally multivariate) input (or explanatory) variable $X$ on a univariate output (or response) variable $Y$. For generalizations see for example \fullciteA{micchelli2005learning} or \fullciteA{caponnetto2007optimal}. We are interested in exploring the functional relationship that describes the conditional distribution of $Y$ given $X$. Technically spoken there is a probability space $(\Omega, \mathcal A, Q)$ which, however, is not of our interest during the further analyses. It is just mentioned to give a technically complete definition of the setting. Fundamental notions of statistics like probability space, random variable, Borel-$\sigma$-algebra etc. will not be defined in this paper; we refer to \fullciteA{Hoffmann-Jorgensen2003VolI} instead. $\mathfrak B_M$ denotes the Borel-$\sigma$-algebra on a set $M$. We only use Borel-$\sigma$-algebras, i.\,e. a measurable set is a Borel-measurable set and a measurable function is measurable with respect to the Borel-$\sigma$-algebras. We consider random variables $X: (\Omega,\mathcal A) \to (\X,\BorelX)$ and $Y: (\Omega,\mathcal A) \to (\Y,\BorelY)$ and their joint distribution $P := (X,Y)\circ Q$ on $(\X\times\Y,\BorelXY)$. $\X$, the so called input space, is assumed as a separable metric space. For notions like metric spaces, separability, Polish spaces etc. we refer to \fullciteA{DunfordSchwartz1958}. The output space $\Y$ is assumed to be a closed subset of the real line $\R$. If $\Y$ is finite, the goal of supervised learning is classification, otherwise it is regression. Considering the ''process'' that, in a first step, the nature creates a realization $x = X(\omega)$ and, after that, in a second step, nature creates the corresponding $y = Y(\omega)$, we are interested in the conditional distribution of $Y$ given $X$ as mentioned above. Since $\Y$ is a closed subset of $\R$, it is a Polish space. Therefore, see \fullciteA[Theorem~10.2.1, p.~343f.; Theorem~10.2.2, p.~345]{Dudley2004}, there is a unique regular conditional distribution of $Y$ given $X=x$ and we can split up the joint distribution $P$ into the marginal distribution $P^X$ on $(\X,\BorelX)$ and the conditional distribution $P(\cdot|x) := P(\cdot|X=x)$. Note that $X$ needs not to be Polish, so no completeness assumptions need to be made on $\X$. 

When we talk about a data set, a sample or observed data, we think (for $n\in\N$) about an $n$-tuple $\mathcal D_n$ of i.i.d. observations, $$\mathcal D_n =  \left(  (x_1, y_1), \ldots, (x_n, y_n) \right) := \mathfrak D_n(\omega) := \left(  (X_1(\omega), Y_1(\omega)), \ldots, (X_n(\omega), Y_n(\omega)) \right) \in (\X\times\Y)^n$$  for $\mathfrak D_n : \Omega \to (\X\times\Y)^n$ the sample generating random variable. Later on we will also  allow the limit $n\to\infty$ in order to prove asymptotic properties of support vector machines. Note that, although it is a tuple, we treat it as a set and use notations like $\in, \cap, ...$. There is only one exception: We allow that the sample contains a data point twice or several times. 

Statistics is done to find a good prediction $f(x)$ of $y$ given a certain $x$. Here, $y$ is the label of the class (or rather its numerical code) in the case of classification, see e.\,g. \fullciteA{christmann2002classification}; or a rank in ordinal regression, see e.\,g. \fullciteA{herbrich1999support}; or a quantile, see e.\,g. \fullciteA{steinwart2011estimating}, the mean (or a value related to it as in the case of the $\varepsilon$-insensitve loss which additionally offers some sparsity properties, see e.\,g. \fullciteA{steinwart2009sparsity}) or expectile (\fullciteA{farooq2015svm}) of the conditional distribution of $Y$ given this certain $x$. Other aims may be ranking (\fullciteA{clemenccon2008ranking}, \fullciteA{agarwal2009generalization}), metric and similarity learning (\fullciteA{mukherjee2006learning}, \fullciteA{xing2003distance}, \fullciteA{cao2016generalization}) or minimum entropy learning (\fullciteA{hu2013learning}, \fullciteA{fan2016consistency}). For $n\in\N$ a function $\mathcal L : (\X\times\Y)^n \to \left\{f:\X\to\R\ |\ f \text{\ measurable}\right\}$ which maps the sample $\mathcal D_n$ at hand to a predictor $f_{\mathcal D_n}$ is called a statistical learning method. Obviously, we are interested in meaningful learning methods which lead to good predictions. It is necessary to define more precisely what a good prediction is. Therefore we present the well-known approach via loss functions and risks. The duty of a loss function is to compare a predicted value with its true counterpart. For different tasks in classification and regression there are different loss functions proposed and discussed, see \fullciteA{rosasco2004loss}, \fullciteA{Steinwart2007} and \fullciteA[Chapter~2, Chapter~3]{steinwart2008support}. Formally, a supervised loss function (or shorter: a loss function) is defined as a measurable function $L:\Y\times\R \to [0,\infty[$. For unsupervised learning a slightly different definition is needed. Since this paper has its focus on supervised learning, we drop the more general definition of loss functions and limit the definition to the presented one. For some technical reasons
 we are also interested in the so called shifted version $\Ls$ of a loss function $L$, defined by $\Ls:\Y\times \R \to \R, \ \Ls(y,t) := L(y,t) - L(y,0)$, see also Appendix \ref{shiftanhang}. If a prediction is exact, we claim the loss function $L$ to be 0, i.\,e. $L(y,y) = 0$ for all $y\in\Y$. All of the common loss functions fulfill this requirement. As the only information about the underlying distribution $P$ is given by the sample $\mathcal D_n$ we cannot expect to find a predictor $f_{\mathcal D_n}$ which provides $L(y,f_{\mathcal D_n}(x)) = 0$ for all $x\in\X, y\in\Y$. This might be possible for all $(x_i,y_i)$ , $i=1,\ldots,n$,  contained in $\mathcal D_n$. But a learning method which provides this property is obviously vulnerable to overfitting --- and overfitting is not desirable when we think about future predictions or slight measurement errors in the sample. Therefore, we want to find a predictor which minimizes the average loss instead. As we are interested in the whole population and not only in the sample, we want to minimize the average loss over all possible $x\in\X$ and all $y\in\Y$. This average loss is called the risk over $\X$ of a measurable predictor $f$ with respect to the chosen loss function $L$ and the unknown underlying distribution $P$ and is formally defined by
$$\mathcal R_{\X,L,P} :  \left\{f:\X\to\R\ |\ f \text{\ measurable}\right\}  \to \R, \ \mathcal R_{\X,L,P}(f) := \int\limits_{\X\times\Y} L(y,f(x))\ dP(x,y).$$
If we use the shifted version of $L$, the definition remains valid:
$$\mathcal R_{\X,\Ls,P}(f) := \int\limits_{\X\times\Y} L(y,f(x)) - L(y,0)\ dP(x,y).$$
Even if all situations were known, we cannot expect that the risk of a measurable predictor with respect to $L$ and $P$ will be 0. Therefore we have to compare our learning method to the best, i.\,e. to the minimal risk which we can reach by using a measurable predictor. For technical reasons (we use integrals) we have to assume measurability. So the best reachable risk is
$$\mathcal R_{\X,L,P}^* := \textup{inf}\left\{\mathcal R_{\X,L,P}(f)\ |\ f:\X\to\R \text{ measurable}\right\},$$
which is called the Bayes risk on $\X$ with respect to $L$ and $P$. Once again, it is possible to define the Bayes risk for the shiftet version of the loss function $L$, too:
$$\mathcal R_{\X,\Ls,P}^* := \textup{inf}\left\{\mathcal R_{\X,\Ls,P}(f)\ |\ f:\X\to\R \text{ measurable}\right\}.$$
It might also be necessary to have a view on the best risk over a smaller class of functions. If $\mathcal F$ is a subset of the measurable functions from $\X$ to $\R$, we define
$$\mathcal R_{\X,L,P,\mathcal F}^* := \textup{inf}\left\{\mathcal R_{\X,L,P}(f)\ |\ f\in\mathcal F\right\}$$
and
$$\mathcal R_{\X,\Ls,P, \mathcal F}^* := \textup{inf}\left\{\mathcal R_{\X,\Ls,P}(f)\ |\ f\in\mathcal F\right\}.$$
If we want to take the mean not over $\X$ but on a measurable subset $\Xi$, the notation will be analogue, for example:
$$\mathcal R_{\Xi,L,P}(f) := \int\limits_{\Xi\times\Y} L(y,f(x))\ dP(x,y).$$

Motivated by the law of large numbers, we use the information contained in the sample to approximate the abovementioned risks. Let
$$D_n := \frac 1 n \sum\limits_{i=1}^n \delta_{(x_i,y_i)}$$
be the empirical distribution based on $\mathcal D_n$ where $\delta_{(x_i,y_i)}$ is the Dirac measure at $(x_i,y_i) \in \X\times\Y$. Note that this empirical distribution has a random aspect since the sample $\mathcal D_n$ is a realization of random variables. By using the empirical measure we can define the empirical risk
$$\mathcal R_{\X, L, D_n}(f) = \frac 1 n \sum\limits_{i=1}^n L(y_i, f(x_i)).$$
If we consider a measurable subset $\Xi$ of $\X$, we get
$$\mathcal R_{\X, L, D_n}(f) = \frac 1 {|\mathcal D_n \cap \Xi|} \sum\limits_{(x_i,y_i)\in\mathcal D_n \cap \Xi} L(y_i, f(x_i)),$$
where $|M|$ denotes the number of elements of a finite set $M$.
To use the information in the sample the SVM is learnt on minimizing the empirical risk. To avoid overfitting, we have to control the complexity of the predictor. Therefore we add a regularization term $p(\lambda, f)$ where $\lambda > 0 $ stands for the influence of the regularization term. In this paper, $p(\lambda, f) := \lambda \|f\|_H^2$ will always be used to avoid overfitting. There are several other regularization terms proposed in literature, in particular for linear support vector machines: for example $\ell_1$-regularization when there is a special view on sparsity, see \fullciteA{zhu20041}; or elastic nets, see \fullciteA{zou2005regularization}, \fullciteA{wang2006doubly}, \fullciteA{de2009elastic}. Other types of regularization, like $\lambda\|f\|_H^q$ for some $q\ge1$ are also possible. Note that $\lambda$ might and should depend on $n$. In case of support vector machines in this paper, $H$ is a so called reproducing kernel Hilbert space of measurable functions. We will have a closer look to reproducing kernel Hilbert spaces later on. Summarizing this, we want to
$$\textup{minimize\ \ } \mathcal R_{\X,L,D_n,\lambda_n} := \frac 1 n \sum\limits_{i=1}^n L(y_i,f(x_i)) + \lambda_n \|f\|_H^2$$
or
$$\textup{minimize\ \ } \mathcal R_{\X,\Ls,D_n,\lambda_n} := \frac 1 n \sum\limits_{i=1}^n \Ls(y_i,f(x_i)) + \lambda_n \|f\|_H^2,$$
based on a sample $D_n$ of observations created by $P$.
Hence, we want to find the empirical support vector machine
$$f_{\Ls,D_n,\lambda_n} := \underset{f\in H}{\arginf} \ \  \frac 1 n \sum\limits_{i=1}^n \Ls(y_i,f(x_i)) + \lambda_n \|f\|_H^2.$$

In practice, the choice of the loss function is determined by the application at hand. On the other side, it is not obvious how to choose the right reproducing Hilbert space. Due to  the  bijection between kernels and their reproducing kernel Hilbert spaces (RKHS) we can reduce the choice of the RKHS to the choice of a suitable kernel.  Let us discuss this a bit more precisely.

Support vector machines (SVMs) and other kernel based machine learning methods need a theoretical setting which is provided by reproducing kernel Hilbert spaces. An introduction and the general theory can be found in \fullciteA{aronszajn1950}, \fullciteA{scholkopf2001learning} and \fullciteA{berlinet2001}. For the purpose of our explanation we repeat some of the basic definitions and results by referring to these references. A kernel is a function $k:\X\times\X \to \R, \ \ (x,x^\prime) \mapsto k(x,x^\prime)$ which is symmetric, i.\,e. $k(x,x^\prime) = k(x^\prime,x)$ for all $x,x^\prime \in \X$, and positive semi-definite, i.\,e. for all $n\in\N:$ $\insumme \jnsumme \alpha_i \alpha_j k(x_i,x_j) \ge 0$ for all $\alpha_1,\ldots,\alpha_n \in \R$ and $x_1,\ldots,x_n \in\X$. It measures the similarity of its two arguments. For theoretical aspects, it is possible and it might be useful to define kernels as complex valued functions, too. 
A kernel $k$ is called a reproducing kernel of a Hilbert space $H$ if $k(\cdot, x) \in H$ for all $x\in\X$ and $f(x) = \langle f,k(\cdot,x)\rangle_H$ for all $x\in\X$ and all $f\in H$.  $\langle \cdot,\cdot\rangle_H$ denotes the inner product of $H$, $\|\cdot\|_H$ denotes the Hilbert space norm of $H$. In this case, $H$ is called the reproducing Hilbert space (RKHS) of $k$. Details about the bijection between kernels and their RKHS see \fullciteA[Moore-Aronszajn Theorem, p.~19]{berlinet2001}.
One of the most important inequalities in reproducing kernel Hilbert spaces for our purposes is given by the following well-known proposition. A proof is shown in \fullciteA[p. 124]{steinwart2008support}.

\begin{prp}\label{dieungleichungschlechthin}
A kernel $k$ is called  bounded if $||k||_\infty := \sup_{\substack{x\in\X}} \ \sqrt{k(x,x)} \ < \infty.$ If and only if the reproducing kernel $k$ of an RKHS $H$ is bounded every $f\in H$ is bounded and for all $f\in H, x\in \X$ there is the inequality
$$|f(x)| = |\langle f, k(\cdot, x)\rangle_H| \le ||f||_H ||k||_\infty.$$ Particularly:
\begin{align}\label{dieungleichung}
 ||f||_\infty\le ||f||_H ||k||_\infty.
\end{align}
\end{prp}

To get a reasonable statistical method, one of the main goals is universal consistency, i.\,e. to be able to prove that 
$$\mathcal R_{\X,\Ls,P} (f_{\Ls,D_n,\lambda_n}) \nto \mathcal R_{\X,\Ls,P}^*\ \ \ \text{in probability w.r.t.}\ P.$$
Support vector machines fulfill this property under week assumptions. For the situations examined in this paper (introduced above) we basically refer to \fullciteA[Theorem~8, p.~315]{christmann2009consistency}. Proofs for consistency in other situations can be found in the abovementioned publications, for example in \fullciteA{fan2016consistency} for the case of minimum error entropy, or in \fullciteA{christmann2012consistency} in the case of additive models.  

% \section{More about loss functions and risks}
Within the next paragraphs, we will recall some definitions and results. A loss function $L$ is called (strictly) convex, if $t\mapsto L(x,y,t)$ is (strictly) convex for all $(x,y)\in\X\times\Y.$ Its shifted version $\Ls$ is called (strictly) convex, if $t\mapsto \Ls(x,y,t)$ is (strictly) convex for all $(x,y)\in\X\times\Y.$ $L$ is called Lipschitz continuous, if there is a constant $|L|_1\in [0,\infty[$ such that for all $(x,y)\in\X\times\Y$ and all $t, s\in \R$, $|L(x,y,t) - L(x,y,s)|\le|L|_1|t-s|$. Analogously, $\Ls$ is called Lipschitz continuous, if there is a constant $|\Ls|_1\in [0,\infty[$ such that for all $(x,y)\in\X\times\Y$ and all $t, s\in \R$, $|\Ls(x,y,t) - \Ls(x,y,s)|\le|\Ls|_1|t-s|$. 

Note that $f_{\X, \Ls,P,\lambda} = f_{\X, L,P,\lambda}$ if $\mathcal R_{\X, L,P}(0) < \infty$. In this case, it is superfluous to work with $\Ls$ instead of $L$, see  \fullciteA[p.~314]{christmann2009consistency}.

As shown in \fullciteA[p.~314, Propositions~2 and~4]{christmann2009consistency} we get the following proposition.

\begin{prp}\label{endlichkeiundlipschitzt}
\begin{enumerate}[label={(\alph*)}]
\item If $L$ is a loss function which is (strictly) convex, then $\Ls$ is (strictly) convex.
\item If $L$ is a loss function which is Lipschitz continuous, then $\Ls$ is Lipschitz continuous with the same Lipschitz constant.
\item If $L$ is a Lipschitz continuous loss function and $f\in L^1(P^X)$, then $-\infty<\mathcal R_{\X,\Ls, P}(f) <\infty$.
\item If $L$ is a Lipschitz continuous loss function and $f\in L^1(P^X)\cap H$, then $\mathcal R_{\X,\Ls, P, \lambda}(f)>-\infty$ for all $\lambda>0$. 
\end{enumerate}
\end{prp}

\begin{prp}
The empirical SVM with respect to $\mathcal R_{\X,L,\frakDn,\lambda}$ and also the empirical SVM with respect to $\mathcal R_{\X,\Ls,\frakDn,\lambda}$    exists and is unique for every $\lambda\in\ ]0,\infty[$ and every data set $\calDn\in (\X\times\Y)^n$ if $L$ is convex, see \fullciteA[Theorem 5.5, p.~168]{steinwart2008support}, and, with respect to $\mathcal R_{\X,\Ls,\frakDn,\lambda}$, the fact that for a given data set $\calDn$ the obviously finite additional term $\frac 1 n \sum\limits_{i=1}^n L(x_i,y_i,0)$ is a constant, see \fullciteA[p. 313]{christmann2009consistency}.

The theoretical SVM exists and is unique  for every $\lambda\in\ ]0,\infty[$ if $L$ is a Lipschitz continuous and convex loss function and $ H\subset L^1 (P^X)$ is the RKHS of a bounded measurable kernel, see \fullciteA[Lemmas 5.1 and 5.2, p.~166f.]{steinwart2008support} and \fullciteA[Theorems~5 and~6, p.~314]{christmann2009consistency}.
\end{prp}

\subsection{Localized approaches and regionalization}\label{local}

The idea of localized statistical learning is not new. Early theoretical investigations are already given by \fullciteA{bottou1992local} and \fullciteA{vapnik1993local}. As pointed out in \fullciteA{hable2013universal} there is a need to have a closer look at localized approaches of machine learning algorithms like support vector machines and related methods. With regard to big data the massive amount of available data potentially used to discover patters is a challenge for algorithms with respect to runtime and storage capacities. Local approaches might help to relieve these issues and are already proposed as experimental studies, see e.\,g. \fullciteA{bennett1998support}, \fullciteA{wu1999large} and \fullciteA{chang2010tree} who use decision trees; for using $k$-nearest neighbor (KNN) samples see \fullciteA{zhang2006svm}, \fullciteA{blanzieri2007instance}, \fullciteA{blanzieri2008nearest} or \fullciteA{segata2010fast}; \fullciteA{cheng2007localized}, \fullciteA{cheng2010efficient}, \fullciteA{gu2013clustered} use KNN-clustering methods to localize the learning. Furthermore, it seems to be straightforward, to parallelize the calculations for different local ares. From a statistical point of view there is another motivation to have a closer look at localized approaches. Different local areas of $\X\times\Y$ might have different claims on the statistical method. For example, there might be a region that requires a simple function serving as predictor for the class or the regression value; another region might need a very volatile function. Global machine learning approaches find their predictors by determining optimal hyperparameters (for example the bandwidth of a kernel or the regularization parameter $\lambda$ as  mentioned above). These parameters determine the complexity of the predictor. By learning the hyperparameters on the whole input space, the optimization algorithm has in some sense to average out the specifics of the local areas. There are at least two possible ways to overcome this disadvantage by using localized approaches. The first one which is examined by \fullciteA{hable2013universal} from a statistical point of view learns the predictor on a point $x\in\X$ as follows: Determine an area around this point (for example a ball, see \fullciteA{zakai2009consistency}; or by a $k$-nearest neighbor approach) and learn the predictor by using the training data within  this area. The second idea is to divide the input space into several, possibly overlapping regions and to learn local predictors on these regions. To get a prediction for a particular $x\in\X$, use the corresponding predictor(s) of the region(s). The goal of this paper is to examine the statistical properties of the latter approach. Note that there are other ways to deal with an high amount of data, see for example \fullciteA{zhangduchi} or \fullciteA{zhouguolin}.

The fact that it is necessarily possible to localize a consistent learning method (because it behaves in a local manner anyway) has already been shown by \fullciteA{zakai2009consistency}. Furthermore,  there is related work to the considered approach investigating optimal learning rates (and therefore also consistency) regarding to partitions like Voronoi partitions (see \fullciteA{Aurenhammer1991}), the least squares or the hinge loss function and the Gaussian kernel under some assumptions on the underlying distribution $P$ and the Bayes decision function, see \fullciteA{eberts2014adaptive}, \fullciteA{meister2016optimal} and \fullciteA{thomann2016spatial}. In contrast, this paper allows more general regionalization methods (with possibly overlapping regions), kernels and loss functions but does not provide any learning rates. However, there is almost no assumption on $P$ in the paper at hand.

%\section{Regionalization}
As usual in statistical learning theory, the given data $\calDN$ is divided by chance into some subsamples. First of all, we need a subsample to train the regionalization method. We denote this subsample by \nomenclature[O]{$\mathcal D^{(R)}$}{Training sample for the regionalization method consisting of $0 < r < N$ observations}$\mathcal D^{(R)}$. Recall that (sub)samples in this work allow that the same data point may appear more than once in a (sub)sample. We define \nomenclature[C]{$r$}{Number of observations in the training sample for the regionalization method, $r=\vert \mathcal D^{(R)} \vert $}$r := |\mathcal D^{(R)}|$.  For all our considerations we presume that $0 < r < N$. We can write: $\mathcal D^{(R)} \in (\X\times\Y)^{r}.$ The independent rest of the data is denoted by $\calDn$ where $n$ is given by $r+n=N$. $\mathcal D_n$ will be used to learn the SVM in the step after the regionalization.

We recall that %\begin{borrow}\label{subsetsofseperable}
subsets of separable metric spaces are separable metric spaces, see \fullciteA[I.6.4, p.~19, and  I.6.12, p.~21]{DunfordSchwartz1958}.
%\end{borrow}

For the aim of localizing the input space $\X$, we need an appropriate regionalization method $R_{\X,r} : (\X \times \Y)^{r} \to \mathcal P(\X)$ where \nomenclature[S]{$\mathcal P(E)$}{Set of subsets of a set $E$}$\mathcal P(\X)$ is the power set of $\X$. The regionalization method need not to get specified, but there are some properties we require:

\begin{itemize}

\item[\textbf{(R1)}] For every fixed (e.\,g. chosen by the user) $r\in \N$ the regionalization method $R_{\X,r}$ divides the input space $\X$ into possibly overlapping regions, i.\,e. $$R_{\X,r}(\mathcal D^{(R)}) = (\X_{(r,1)}, \ldots, \X_{(r,B_{r})}) = (\X_{(r,b)})_{b=1,\ldots,B_{r}}$$  with $\X = \bigcup\limits_{b=1}^{B_{r}}\X_{(r,b)}.$ $B_{r}$ is the number of regions, usually chosen by the regionalization method and depends therefore on the subsample  $\mathcal D^{(R)}$. Note that $B :=B_{r}$ is constant after the regionalization.\\
%\textbf{or}

\item[\textbf{(R2)}] For every fixed and given $r\in \N$ and for every $b\in\left\{1,\ldots,B_{r}\right\}$ ${\X_{(r,b)}}$ is a metric space and, in addition, a complete measurable space, i.\,e. for all probability measures is $(\X_{(r,b)}, \mathfrak B_{{\X_{(r,b)}}})$ complete (see for example \fullciteA[Definition~1.3.7, p.~17]{ash2000probability}).\\

\item[\textbf{(R3)}] For $r\to\infty$, the regionalization method ensures $\left|\X_{(r,b)}\right |\to \infty$  for all $b\in\left\{1,\ldots,B_{r}\right\}$, i.\,e. $$\lim_{\substack{n\to\infty}}\  \min_{b\in\left\{1,\ldots,B_r\right\}}\left|\X_{(r,b)}\right |= \infty.$$

\end{itemize}

Hence, after the regionalization, the number of the regions and the regions themselves are fixed. So we can simplify our notation and look at $$%\X = \bigdcup\limits_{b=1}^{B} \Xb\ \ \ \ \ \text{or}\ \ \ \ \ 
\X = \bigcup\limits_{b=1}^{B} \Xb.$$

In this and in the following sections, we will use the  notation $$\XI := \left(\ \bigcap\limits_{b\in I} \Xb\  \right)\ohne\left(\ \bigcup\limits_{b\not\in I} \Xb\ \right), \ \ \ \ I\subset \btoB.$$  By $\|\cdot\|_{\Xb\text{-}\infty}$ we will denote the supremum norm on $\Xb$, i.\,e. for a function $f:\X\to\R$ we define $$\|f\|_{\Xb\text{-}\infty} := \sup\left\{f(x)| x\in\Xb\right\}$$ and for a kernel $k:\X\times\X\to\R$ we define
$$\|k\|_{\Xb\text{-}\infty} := \sup_{\substack{x\in\Xb}} \ \sqrt{k(x,x)}.$$

Assume that the whole input space $\X$ will be divided by a regionalization method in some regions $\X_1,\ldots,\X_B$ which need not to be disjoint. Then we will learn the SVMs separately, one SVM for each region. After that we combine these local SVMs to a composed estimator or classificator, respectively, delivering reasonable values for all $x\in\X$. The influence of the local predictors is pointwise controlled by measurable weight functions $w_b : \X \to [0,1], b\in\btoB$, which fulfill for all $x\in\X$ the following conditions:
\begin{itemize}
\item [\textbf{(W1)}] $\sum\limits_{b = 1}^B w_b(x)  = 1$ for all $x\in\X,$
\item [\textbf{(W2)}] $w_b(x) = 0$ for all $x\notin\Xb$  and for all $b\in\btoB.$
\end{itemize}

Therefore, our composed predictors are defined as

\nomenclature[F]{$\fLPlambdacomp$}{Theoretical composed predictor in the overlapping case}
\nomenclature[F]{$\fLDnlambdacomp$}{Empirical composed predictor in the overlapping case}

\begin{align}%\label{konstruktionvonfLPlambdacomp}
\fLPlambdacomp : \X\to\R, \ \ \ \ \  \fLPlambdacomp(x) :=\zusammentheoretischx,     \notag
\end{align}
\begin{align}%\label{konstruktionvonfLDnlambdacomp}
\fLDnlambdacomp : \X\to\R, \ \ \ \ \  \fLDnlambdacomp(x) := \zusammenempirischx,     \notag
\end{align}

where

\begin{itemize}

\item $P$ is the unknown distribution of $(X,Y)$ on $\X\times\Y$ and $\frakDn := \frac 1 n \insumme \delta_{(x_i, y_i)}$ is the empirical measure based on a sample or data set $\calDn := \left(  (x_1, y_1), \ldots, (x_n, y_n) \right)$ of $n$ i.i.d. realizations of $(X,Y)$.
\item $\Pb$ is the theoretical distribution on $\Xb\times \Y$, $\frakDnb$ its empirical analogon.  They are in fact probability distributions in all interesting situations, i.\,e. if $ P(\Xb\times \Y) > 0$ or $\frakDn(\Xb\times \Y) > 0$, respectively, because they are built from $P$ and $\frakDn$ as follows:

$$ \Pb := \begin{cases}
 \ \  \frac 1 {P(\Xb\times \Y)} \ P_{|_{\Xb\times \Y}} \ \   &, \  \text{if} \ \  P(\Xb\times \Y) > 0         \\
  \ \ \ \ \ \ \ \ 0                                                                 &, \  \text{otherwise}
\end{cases}$$
and
$$\nomenclature[D]{$\frakDnb$}{Empirical distribution on $\Xb\times\Y$} \frakDnb := \begin{cases}
 \ \ \frac 1 {\frakDn(\Xb\times \Y)}\  \frakDn_{|_{\Xb\times \Y}}  \ \   &, \  \text{if} \ \  \frakDn(\Xb\times \Y) > 0         \\
   \ \ \ \ \ \ \ \ 0                                                                                    &, \  \text{otherwise}
\end{cases}.$$
The zero denotes in this definitions the null measure on $\BorelXnbY$. If $\Xb\times\Y$ is a null set with respect to $P$ or $\frakDn$, respectively, we are not interested in this region. This motivates the decision for the 0. Otherwise we want to deal with a probability measure which motivates the weights ${\frakDn(\Xb\times \Y)}^{-1}$ and ${P(\Xb\times \Y)}^{-1}$.  Of course, this definition is a bit technical but very practical.

Note that $ P_{|_{\Xb\times \Y}}:=  \Ind_{\Xb\times\Y} P$ and $ \frakDn_{|_{\Xb\times \Y}}:= \Ind_{\Xb\times\Y}\frakDn$ are the restrictions of $P$ and $\frakDn$ on (the Borel-$\sigma$-algebra on) $\Xb\times \Y$. % i.\,e. for all $g\in L^1(P)$ 

The collection of all data points in $\Xb\times \Y$ will be denoted by $\calDnb$ which is of course a subtuple of $\calDn$. With this notation we can write $\frakDn(\Xb\times \Y) =|\calDnb| =:  n_b$. \nomenclature[C]{$n_b$}{Number of observations in $\Xb\times\Y$, $n_b = \vert \calDnb \vert$}
\item In an analogous way,  the regional marginal distribution of $X$ is defined by $\PbX := P^X(\Xb)^{-1} P^X_{|\Xb}$ if $P^X(\Xb)>0$ and 0 else.
\item $\lambda := (\lambda_1, \ldots, \lambda_B) \subset\  ]0, \infty[^{B}.$ Later on, we will also consider $$(\lambda_n)_{n\in\N} := (\lambda_{(n_1,1)},\ldots, \lambda_{(n_B,B)})_{n\in\N}, \ \ \ \ \ n = \bsumme n_b,$$ instead of a fixed $\lambda$.
\item  By $\lokaltheoretisch$ we denote the theoretical local SVM learnt on $\Xb\times\Y$ with respect to $\Ls$ and $\Pb$, if $\Pb$ is a probability measure; if $\Pb$ is the null measure, $\lokaltheoretisch$ is an arbitrary measurable function. By $\lokalempirisch$ we denote the empirical local SVM learnt on $\Xb\times\Y$ with respect to $\Ls$ and $\frakDnb$, if $\frakDnb$ is a probability measure; if $\frakDnb$ is the null measure, $\lokalempirisch$ is an arbitrary measurable function.

\end{itemize}

Note that in case of overlapping regions these composed predictors need not to be elements of any Hilbert space and, therefore, the expression $\left\|\fLPlambdacomp\right\|_H$ does not make any sense.

\section{Results}\label{results}

\subsection{$\Ls$-risk consistency}

This section contains the main results of the paper: consistency and robustness properties of the composed predictors that have been learnt locally. The results hold for all distributions, i.\,e. also for heavy tailed distributions like Cauchy or stable distributions.

\begin{thm}\label{consistencytheorem}

Let $L$ be a convex, Lipschitz-continuous (with Lipschitz-constant $|L|_1\neq 0$) loss function and $\Ls$ its shifted version.
For all $b\in\btoB$ let $k_b$ be a measurable and bounded kernel on $\X$ and let the corresponding RKHSs $H_b$ be separable. Let the regionalization method fulfill \textbf{(R1), (R2),} and \textbf{(R3)}.

Then for \textbf{all} distributions $P$ on $\X\times\Y$ with $H_b$ dense in $L^1(\PbX), b\in\btoB,$ and every collection of sequences $\lambda_{(n_1,1)},\ldots,\lambda_{(n_B,B)}$ with $\lambdanb\to 0$  and $\lambda_{(n_b,b)}^{2}n_b \to\infty$ when $n_b\to\infty$, $b\in\btoB,$ it holds true that
$$\mathcal R_{\X,\Ls,P}(\fLDnlambdancomp) \nto \mathcal R_{\X,\Ls,P}^* \ \ \ \text{in probability with respect to $P$}.$$
\end{thm}

\begin{bem}
\begin{enumerate}[label={(\alph*)}]
\item Note that the assumption $|L|_1\neq 0$ is purely technical when we have a look at the use of SVMs. If $|L|_1 = 0$, the loss function $L$ would be constant and therefore not be interesting for statistical learning from an applied point of view.
\item The denseness assumption is not very strict. For example, the Gaussian-RBF-kernel fulfills the assumption (see \fullciteA[Theorem 4.63, p.~158]{steinwart2008support}).
\item By using the Lipschitz-continuity of $\Ls$ and under the given assumption that $H_b$ dense in $L^1(\PbX), b\in\btoB,$ it holds for all $n\in\N$ that
\begin{align}\notag
& \left| \mathcal R_{\X,\Ls,P}\left(\fLDnlambdancomp(x)\right)\right|\notag = \left|\  \int\limits_{\X\times\Y} \Ls\left(y,\fLDnlambdancomp(x)\right) \ dP(x,y)\ \right| \notag \\ 
& = \left|\  \int\limits_{\X\times\Y} \Ls\left(y,\zusammenempirischnx\right) \ dP(x,y)\ \right|\notag \\
& = \left|\  \int\limits_{\X\times\Y} L\left(y,\zusammenempirischnx\right) - L(y,0) \ dP(x,y)\ \right|\notag \\
& \le |L|_1 \  \int\limits_{\X\times\Y} \left|\zusammenempirischnx\right| \ dP(x,y)\ \notag \\
& \le |L|_1 \  \Isumme\ \int\limits_{\XI\times\Y} \left|\zusammenempirischnx\right| \ dP(x,y)\ \notag \\
& \le |L|_1 \  \Isumme\ \bsumme\  \int\limits_{\XI\times\Y} \left|\lokalempirischn(x)\right| \ dP(x,y)\ \notag \\
& \le |L|_1 \  \Isumme\ \bsumme\  \int\limits_{\Xb\times\Y} \left|\lokalempirischn(x)\right| \ dP(x,y)\ \notag \\
& \le |L|_1 \  \Isumme\ \bsumme\  P(\Xb\times\Y) \int\limits_{\Xb} \left|\lokalempirischn(x)\right| \ d\PbX(x)\ < \infty. \notag 
\end{align}\notag
Therefore, problems concerning infinite risks cannot arise in our situation. The estimates remain true when we consider the theoretical SVMs, i.\,e. when we replace $D_n$ by $P$.
\item The assumption that the RKHS $H$ is separable is not difficult to fulfill. The usage of a continuous kernel guarantees the separability of $H$, see  \fullciteA[Lemma 4.33, p.~130]{steinwart2008support}.
\end{enumerate}
\end{bem}

As shown above, consistency holds for arbitrary weights 
$w_b : \X \to [0,1], b\in\btoB$, fulfilling $\sum\limits_{b = 1}^B w_b(x)  = 1$ 
for all $x\in\X$, and $w_b(x) = 0$ for all $x\not\in \Xb$.  The user can decide which weights he wants to apply. We propose two simple weight functions.

Firstly, the  weight
\begin{align}
w_b(x) = \frac {\Ind_{\X_b}(x)} { \sum\limits_{\beta = 1}^B \Ind_ {\X_\beta}(x)}, \ \ \ b\in\btoB,\notag
\end{align}
which guarantees that every local SVM comes only in its own region into operation, in intersections we then use the average of the relevant local SVMs. If a weighted average is needed,  the weights can be defined as 
\begin{align}
w_b(x) =  { \sum\limits_{b = 1}^B\  \theta_b\ \Ind_ {\Xb}(x)}, \ \ \ b\in\btoB,\notag
\end{align}
 for suitable $\theta_b$, i.\,e. $\theta_b\in[0,1]$ such that  $\bsumme w_b(x) = 1$ for all $x\in\X$.
  
%\ERG[evtl. emprische Untersuchungen zu den Gewichten]

\subsection{Robustness}

In addition to consistency, we can show a robustness property. Let $\mathcal M_1(M)$ denote the set of probability measures on the Borel-$\sigma$-algebra of a set $M$. For all $b\in\btoB$ and $\varepsilon_b\in [0,\frac 1 2[$ let again $\Pb := P(\Xb\times\Y)^{-1}P_{|\Xb\times\Y}$ if $P(\Xb\times\Y)\neq 0$ and the null measure otherwise. Define for a distribution $P$ on $\X\times\Y$ the $\varepsilon_b$-contamination environment of $P$ (or $\Pb$) on $\Xb\times\Y$
$$N_{b,\varepsilon_b}(P) := N_{b,\varepsilon_b}(\Pb) := \left\{(1-\varepsilon_b)\Pb + \varepsilon_b\tilde\Pb\ \left|\ \tilde\Pb\in\mathcal M_1(\Xb\times\Y) \right.\right\}.$$
Furthermore, let $\varepsilon := (\varepsilon_1,\ldots,\varepsilon_B)\in [0,\frac 1 2[^B$. By using these notations, we can define
$$N_\varepsilon(P) := \left\{\tilde P\in \mathcal M_1(\X\times\Y) \ \left|\ \tilde\Pb\in N_{b,\varepsilon_b}\ \text{ for all } b\in\btoB \right.\right\}$$
as a composed $\varepsilon$-contamination environment of $P$ on $\X\times\Y$. 

It is then possible to give an upper bound for the maxbias, i.\,e. an upper bound for the possible bias of the composed predictor based on locally learnt SVMs. Recall that we denote the supremum norm on $\Xb$ by $\|\cdot\|_{\Xb\text{-}\infty}$, i.\,e. for a function $f:\X\to\R$ we define $$\|f\|_{\Xb\text{-}\infty} := \sup\left\{f(x)| x\in\Xb\right\}$$
and for a kernel $k:\X\times\X\to\R$ we define
$$\|k\|_{\Xb\text{-}\infty} := \sup_{\substack{x\in\Xb}} \ \sqrt{k(x,x)}.$$

\begin{thm}\label{robustnesstheorem}
Let $L$ be a convex, Lipschitz-continuous (with Lipschitz-constant $|L|_1\neq 0$) loss function and $\Ls$ its shifted version.
For all $b\in\btoB$ let $k_b$ be a measurable and bounded kernel on $\X$ and let the corresponding RKHSs $H_b$ be separable. Let the regionalization method fulfill \textbf{(R1), (R2)} and \textbf{(R3)}. Then, for all distributions $P$ on $\X\times\Y$ and all $\lambda := (\lambda_1,\ldots,\lambda_B)\in [0,\infty[^B$, it holds that
$$\underset{\tilde P \in N_{\varepsilon}(P)}{\sup} \ \left\|\fLPschlangelambdacomp - \fLPlambdacomp\right\|_{\infty}\le 2\ |L|_1 \ \bsumme\ \|w_b\|_{\Xb\text{-}\infty}\   \frac {\varepsilon_b}{\lambda_b}\  \|k_b\|^2_{\Xb\text{-}\infty}.$$
\end{thm}

Note that this bound is a uniform bound in the sense that it is valid for \textbf{all} distributions $P$ and all weighting schemes fulfilling \textbf{(W1)} and \textbf{(W2)}, i.\,e. $\sum\limits_{b = 1}^B w_b(x)  = 1$  for all $x\in\X$ and \ $w_b(x) = 0$ for all $x\notin\Xb$ and for all $b\in\btoB.$

\begin{bsp}
Let $d\in\N$, $\X\subset \R^d$, $k_b$ be a Gaussian-RBF-kernel, i.\,e. $k_b(x,x^\prime) := \exp\left(-\gamma_b^{-2} \|x-x^\prime\|_2^2\right), \gamma_b>0$, for all $b\in\btoB$, and let $L$ be the hinge loss (for classification) or the $\tau$-pinball loss (for quantile regression, $\tau \in (0,1)$). Then Theorem~\ref{robustnesstheorem} provides the upper bound 
$$\underset{\tilde P \in N_{\varepsilon}(P)}{\sup} \ \left\|\fLPschlangelambdacomp - \fLPlambdacomp\right\|_{\infty} \le \ \bsumme \frac{1}{\lambda_b}.$$
\end{bsp}

\section{Summary}\label{summary}

By proving universal consistency and robustness (in the sense of a bounded maxbias) of locally learnt predictors we have shown that learning in a local way conserves desirable properties of kernel based methods like support vector machines. We see that there is no disadvantage of learning separate predictors, one for each region, and combining them from these point of views. These results were shown for all distributions and only under assumptions which are verifiable by the user.

\section*{Acknowledgement}

The author would like to thank Andreas Christmann and Katharina Strohriegl for useful discussions.

\begin{appendix}

\section{Proofs}

It might be useful to recall inequality (\ref{dieungleichung}) for functions $f$ in the reproducing kernel Hilbert space of a bounded kernel $k$: $ ||f||_\infty\le ||f||_H ||k||_\infty.$

For the proof of Theorem~\ref{consistencytheorem} we need the following lemma.

\begin{lem}\label{bayesinhilbert}
To obtain the Bayes risk with respect to a shifted loss function $\Ls$ and a distribution $P$, it is sufficient to optimize over all $f\in \mathcal F$, if $\mathcal F$ is dense in $L^1(P^X)$, i.\,e. 
$$ \mathcal R_{\X,\Ls,P}^* =  \mathcal R_{\X, \Ls,P,\mathcal F}^*.$$
\end{lem}

\begin{proof}[Proof of Lemma~\ref{bayesinhilbert}]
We see that
\begin{align}
 \mathcal R_{\X,\Ls,P}^* & = \textup{inf} \left\{\ \left. \int\limits_{\X\times\Y} \Ls(y,f(x))\ dP(x,y) \right| f:\X\to\R, f \text{ measurable }\right\}  \notag\\
& =  \textup{inf} \left\{\ \left. \int\limits_{\X\times\Y} L(y,f(x)) - L(y,0) \ dP(x,y) \right| f:\X\to\R, f \text{ measurable }\right\}  \notag\\
& =  \textup{inf} \left\{\ \left. \int\limits_{\X\times\Y} L(y,f(x)) - L(y,0) \ dP(x,y) \right| f:\X\to\R, f\in \mathcal F\right\} \notag\\
& =   \mathcal R_{\X,\Ls,P,\mathcal F}^*.\notag
\end{align}
This is true due to \fullciteA[Theorem~5.31, p.~190]{steinwart2008support} and the fact, that there is no influence of $f$ on $L(\cdot, 0)$.

\end{proof}

\begin{proof}[Proof of Theorem~\ref{consistencytheorem}]

To get started, we decompose the difference of risks into three parts. To be able to do this we firstly have to approximate the Bayes risk.

For $\mathcal R_{\X,\Ls,P}^* = \inf\left\{\left.\mathcal R_{\X,\Ls,P}(f)\ \right|\ f:\X\to\R \text{ measurable} \right\}$ there exists (by the definition of the infimum) a sequence $(f_m^a)_{m\in\N}\subset \left\{ f:\X\to\R \text{ measurable} \right\}$ fulfilling $$\mathcal R_{\X,\Ls,P}(f_m^a) \mto \mathcal R_{\X,\Ls,P}^*.$$ Particularly, for all $\varepsilon>0$, there are indices $m_1,\ldots,m_B  \in\N$ such that $\left|\mathcal R_{\Xb,\Ls,\Pb}(f_{m_b}^a) - \mathcal R_{\Xb,\Ls,\Pb}^* \right| < \frac \varepsilon B$ for all $b\in\btoB$. Let $f^a := f_{\tilde m}^a$ for an $\tilde m \ge \max\left\{m_1,\ldots,m_B\right\}.$

Let us now have a view on the decomposition of the risks. We have
$$\left|\mathcal R_{\X,\Ls,P}(\fLDnlambdancomp) -  \mathcal R_{\X,\Ls,P}^*\right|$$
$$ \le \underbrace{\left|\mathcal R_{\X,\Ls,P}(\fLDnlambdancomp) -  \mathcal R_{\X,\Ls,P}(\fLPlambdancomp)\right|}_{\text{term 1}} + \underbrace{\left|\mathcal R_{\X,\Ls,P}(\fLPlambdancomp) -  \mathcal R_{\X,\Ls,P}(f^a)\right|}_{\text{term 2}} + \underbrace{\left|\mathcal R_{\X,\Ls,P}(f^a) -  \mathcal R_{\X,\Ls,P}^*\right|}_{\text{term 3}} .$$
Term 1 will be examined by stochastic methods, particularly by Hoeffding's concentration inequality on Hilbert spaces. The second term vanishes asymptotically which can be shown by applying the so called approximation error function. The convergence of term 3 to 0 follows directly from the definition of $f^a$.  By these arguments, we can prove stochastic convergence to 0.

For term 1 we find the following bound:

\begin{align}
&\;   \; \; \ \left|\mathcal R_{\X,\Ls,P}(\fLDnlambdancomp) -  \mathcal R_{\X,\Ls,P}(\fLPlambdancomp)\right| \label{consistency_term1} \\
& = \left|\  \int\limits_{\X\times\Y} \Ls\left(y,\fLDnlambdancomp(x)\right) - \Ls\left(y,\fLPlambdancomp(x)\right)\ dP(x,y) \ \right| \notag \\
& = \left|\  \int\limits_{\X\times\Y} L\left(y,\fLDnlambdancomp(x)\right) - L(y,0) - L\left(y,\fLPlambdancomp(x)\right) + L(y,0)\ dP(x,y) \ \right| \notag \\
& \le |L|_1 \int\limits_{\X\times\Y} \left|\fLDnlambdancomp(x) - \fLPlambdancomp(x)\right| \ dP(x,y)  \notag \\
& \le |L|_1    \int\limits_{\X\times\Y} \ \bsumme\ w_b(x)\  \left| \lokalempirischn(x) - \lokaltheoretischn(x) \right| \ dP(x,y) \notag \\
& = |L|_1  \ \bsumme\   \int\limits_{\X\times\Y} \ w_b(x)\  \left| \lokalempirischn(x) - \lokaltheoretischn(x) \right| \ dP(x,y) \notag \\
& = |L|_1  \ \bsumme\   \int\limits_{\Xb\times\Y} \ w_b(x)\  \left| \lokalempirischn(x) - \lokaltheoretischn(x) \right| \ dP(x,y) \notag \\
& \le |L|_1 \bsumme \ P(\Xb\times\Y)\ \| \lokalempirischn - \lokaltheoretischn\|_{\Xb\text{-}\infty}\notag\\
& \stackrel{\text{\small (\ref{dieungleichung})}}{\le} |L|_1 \bsumme \ P(\Xb\times\Y)\  \|k_b\|_{{\Xb\text{-}\infty}} \    \| \lokalempirischn - \lokaltheoretischn\|_{H_b}.\notag
\end{align}

For all $b\in\btoB$, the last factor $\left|\left| \lokalempirischn -  \lokaltheoretischn\right|\right|_{H_{b}}$ converges to $0$ in probability (with respect to $P_b$) if $n_b\to\infty$, which is guaranteed if $n\to\infty$ by \textbf{(R3)}. Hence, the whole expression and therefore the difference in (\ref{consistency_term1}) converges to $0$ in probability (with respect to  $P$). This can be shown as follows. For every $b\in\btoB$ \fullciteA[Theorem~7]{christmann2009consistency} guarantees for every $n\in\N$ the existence of a bounded, measurable function $h_{(n,b)}:\Xb\times\Y\to \R$ with $||h_{(n,b)}||_{\Xb\text{-}\infty} \le |L|_1$ and, furthermore, for every probability measure $\widetilde \Pb$ on $\Xb\times\Y$:
$$\left|\left|f_{b,\Ls,\Pb,\lambdanb} - f_{b,\Ls,\widetilde\Pb,\lambdanb}\right|\right|_{H_{b}} \le \frac 1 \lambdanb \left|\left| \E_\Pb\left[ h_{(n,b)}\Phi_b\right] - \E_{\widetilde\Pb}\left[ h_{(n,b)}\Phi_b\right]\right|\right|_{H_{b}},$$
where $\Phi_b:\Xb\to H_{b}, \ \Phi_b(x) = k_b(\cdot,x)$ for all $x\in\Xb.$ Let $\varepsilon_b\in \ ]0,1[.$ For this proof, we are particularly interested in $\widetilde \Pb := \frakDnb$.  Let
$$S_ {(n,b)} := \left\{\calDnb\in (\Xb\times\Y)^\nb \left|   \left|\left| \E_\Pb\left[ h_{(n,b)}\Phi_b\right] - \E_{\frakDnb}\left[ h_{(n,b)}\Phi_b\right]\right|\right|_{H_{b}} \le \frac {\lambdanb\varepsilon_b}{|L|_1}   \right.\right\},$$
where $\nb := |\frakDnb|$. 
For all $\calDnb\in S_ {(n,b)}$ it follows that
$$ \left|\left| \lokalempirischn -  \lokaltheoretischn\right|\right|_{H_{b}} \le \varepsilon_b.$$
Let us now have a look at the probability of $ S_ {(n,b)}$. As shown in  \fullciteA[p.~323]{christmann2009consistency} by using Hoeffding's inequality in Hilbert spaces, $\Pb^\nb ( S_ {(n,b)} ) \to 1$ for $\nb\to\infty$. Hence,
\begin{align}
\left|\left| \lokalempirischn -  \lokaltheoretischn\right|\right|_{H_{b}}\to 0 \ \ \text{in probability (with respect to  $\Pb$)}, \ \ \nb\to\infty.\notag
\end{align}

We now know for all $b\in\btoB$ that $\left|\left| \lokalempirischn - \lokaltheoretischn \right|\right|_{H_{b}}\to 0$ in probability with respect to $\Pb$ when $\nb\to\infty$. Therefore, term 1 vanishes, i.\,e. $$|\mathcal R_{\X,\Ls,P}(\fLDnlambdancomp) -  \mathcal R_{\X,\Ls,P}(\fLPlambdancomp)| \nto 0$$ in probability with respect to $P$.

\vspace*{0.5cm}

%%%%%%%%%%%%%%%%%%%%%%%%%%%%
%%%%%%%%%%%%%%%%%%%%%%%%%%%%

It remains the investigation of the asymptotic behaviour of the second term where we use the convexity of $L$:

\begin{align}
&\ \ \ \,  \left|\mathcal R_{\X,\Ls,P}(\fLPlambdancomp) -  \mathcal R_{\X,\Ls,P}(f^a)\right| \notag\\
& = \left| \ \int\limits_{\X\times\Y} \ L\left(y, \fLPlambdancomp(x)\right) - L\left(y,f^a(x)\right)\ dP(x,y)\right| \notag\\
& = \left| \ \int\limits_{\X\times\Y} \ L\left(y, \bsumme w_b(x) \lokaltheoretischn(x)\right) - L\left(y,f^a(x)\right)\ dP(x,y)\right|\notag\\
& = \left| \ \int\limits_{\X\times\Y} \ L\left(y, \bsumme w_b(x) \lokaltheoretischn(x)\right) - \underbrace{\left[ \bsumme w_b(x) \right]}_{ = 1}  L\left(y,f^a(x)\right)\ dP(x,y)\right|\notag\\
& \stackrel{\text{\small $L$ convex}} {\le}  \left| \ \int\limits_{\X\times\Y} \ \bsumme w_b(x)  L\left(y, \lokaltheoretischn(x)\right) -  \bsumme w_b(x)   L\left(y,f^a(x)\right)\ dP(x,y)\right|\notag
\end{align}
\begin{align}
& = \left| \Isumme\ \int\limits_{\XI\times\Y} \ \bsumme w_b(x)  \left(L(y, \lokaltheoretischn(x)) -   L\left(y,f^a(x)\right)\right)\ dP(x,y)\right|\notag\\
& = \left| \Isumme\ \bsumme\ \int\limits_{\XI\times\Y} \  w_b(x)  \left(L(y, \lokaltheoretischn(x)) -   L\left(y,f^a(x)\right)\right)\ dP(x,y)\right|\notag\\
& \stackrel{\text{\small (W2)}}{=} \left| \Isumme\ \sum\limits_{b\in I}\ \int\limits_{\XI\times\Y} \  w_b(x)  \left(L(y, \lokaltheoretischn(x)) -   L\left(y,f^a(x)\right)\right)\ dP(x,y)\right|\notag\\
& \le \left| \Isumme\ \sum\limits_{b\in I}\ \int\limits_{\Xb\times\Y} \  w_b(x)  \left(L(y, \lokaltheoretischn(x)) -   L\left(y,f^a(x)\right)\right)\ dP(x,y)\right|\notag\\
& \le  \Isumme\ \sum\limits_{b\in I}\  P(\Xb\times\Y) \left|  \mathcal R_{\Xb,L,\Pb}(\lokaltheoretischn) - \mathcal R_{\Xb,L,\Pb}(f^a)    \right|\notag\\
& \le  \Isumme\ \sum\limits_{b\in I}\  P(\Xb\times\Y) \left|  \mathcal R_{\Xb,L,\Pb}(\lokaltheoretischn) - \mathcal R_{\Xb,L,\Pb}^*    \right|.\notag
\end{align}

Looking at the differences of these local risks, we define, for every $b\in\btoB$ and for every $f\in H_b$, the so called approximation error function 
\begin{align}\label{gleichheitapprox}
A_{b,f}: [0,\infty[ \to\R, \ \ \ \lambda \mapsto \mathcal R_{\Xb,\Ls,\Pb,\lambdab}(f) - \mathcal R_{\Xb,\Ls,\Pb,H_b}^* = \mathcal R_{\Xb,\Ls,\Pb}(f) + \lambdab||f||^2_{H_b} - \mathcal R_{\Xb,\Ls,\Pb,H_b}^*.
\end{align}

From \fullciteA[Lemma A.6.4, p.~520]{steinwart2008support} we know that, if we consider sequences $(\lambdanb)_{\nb\in\N}$ converging to 0, $\lambdanb>0$ for all $n_b \in\N$ and for all $b\in\btoB$, $${\underset{f\in H_b}{\inf}} A_{b,f}(\lambdanb) \to {\underset{f\in H_b}{\inf}} A_{b,f}(0) \ \ \ \text{for $\nb\to\infty$.}$$ As the SVM is defined as a minimizer of the regularized risk, for every $n\in\N, b\in\btoB$,
$${\underset{f\in H_b}{\inf}} A_{b,f}(\lambdanb) = \mathcal R_{\Xb,\Ls,\Pb}(\lokaltheoretischn) + \lambdanb||\lokaltheoretischn||^2_{H_b} -  \mathcal R_{\Xb,\Ls,\Pb,H_b}^*.$$
Due to \fullciteA[Theorem 5.31, p.~190]{steinwart2008support} and (\ref{gleichheitapprox}) we know that ${\underset{f\in H_b}{\inf}} A_{b,f}(0) = 0.$ Hence, we obtain
\begin{align}
0 & \le  \underset{\nb\to\infty}\limsup\left(  \mathcal R_{\Xb,\Ls,\Pb,\lambdab}(\lokaltheoretischn) -  \mathcal R_{\Xb,\Ls,\Pb,H_b}^*\right)  \notag  \\ 
   & =  \underset{\nb\to\infty}\limsup\left(  \mathcal R_{\Xb,\Ls,\Pb}(\lokaltheoretischn) + \lambdanb||\lokaltheoretischn||^2_{H_b} -  \mathcal R_{\Xb,\Ls,\Pb,H_b}^*\right) \notag  \\ 
   & \le \underset{\nb\to\infty}\limsup\left( {\underset{f\in H_b}{\inf}} A_{b,f}(\lambdanb) - {\underset{f\in H_b}{\inf}} A_{b,f}(0)\right) = 0.\notag
\end{align}
Hence, the differences to the smallest risk over $H_b$ vanish. We know that according to Lemma \ref{bayesinhilbert} for every $b\in\btoB$ it holds that $ \mathcal R_{\Xb,\Ls,\Pb}^* =  \mathcal R_{\Xb,\Ls,\Pb,H_b}^*.$
Using this fact, we get that all the local differences to the Bayes risk vanish. Hence, $|\mathcal R_{\X,\Ls,P}(\fLPlambdancomp) -  \mathcal R_{\X,\Ls,P}^*|   \to 0$ for $n\to\infty$.

\vspace*{0.5cm}

Summarizing our results to the terms 1 to 3, $|\mathcal R_{\X,\Ls,P}(\fLDnlambdancomp) -  \mathcal R_{\X,\Ls,P}^*| \to 0$ in probability (with respect to the unknown distribution $P$).
\end{proof}

\begin{proof}[Proof of Theorem~\ref{robustnesstheorem}]
The result is shown as follows, only using the triangle inequality, upper bounds for suprema, the assumed properties of the weights of the local SVMs and a well-known inequality in reproducing kernel Hilbert spaces, see Proposition~\ref{dieungleichungschlechthin}. We denote the total variation norm, see for example \fullciteA[p.~158]{DenkowskiMigorskiPapageorgiou2003}, on $\Xb$ by $\|\cdot\|_{\Xb\text{-TV}}$. Then we get
\begin{align}
\underset{\tilde P \in N_{\varepsilon}(P)}{\sup} \ \left\|\fLPschlangelambdacomp - \fLPlambdacomp\right\|_{\infty} & = \underset{\tilde P \in N_{\varepsilon}(P)}{\sup} \ \underset{x\in\X}{\sup}\  \left| \bsumme\ w_b(x)\left(\lokaltheoretischschlange(x)-\lokaltheoretisch(x)\right)\right| \notag \\
&\le \underset{\tilde P \in N_{\varepsilon}(P)}{\sup} \ \underset{x\in\X}{\sup}\  \bsumme\  \left| w_b(x)\left(\lokaltheoretischschlange(x)-\lokaltheoretisch(x)\right)\right| \notag \\
&\le \underset{\tilde P \in N_{\varepsilon}(P)}{\sup} \ \bsumme\   \underset{x\in\X}{\sup}\  \left| w_b(x)\left(\lokaltheoretischschlange(x)-\lokaltheoretisch(x)\right)\right| \notag \\
&\le \underset{\tilde P \in N_{\varepsilon}(P)}{\sup} \ \bsumme\ \|w_b\|_{\Xb\text{-}\infty}\   \underset{x\in\Xb}{\sup}\  \left| \lokaltheoretischschlange(x)-\lokaltheoretisch(x)\right| \notag \\
&= \underset{\tilde P \in N_{\varepsilon}(P)}{\sup} \ \bsumme\ \|w_b\|_{\Xb\text{-}\infty}\   \left\| \lokaltheoretischschlange-\lokaltheoretisch\right\|_{\Xb\text{-}\infty} \notag \\
&\le \bsumme\  \|w_b\|_{\Xb\text{-}\infty}\  \underset{\tilde P \in N_{\varepsilon}(P)}{\sup} \    \left\| \lokaltheoretischschlange-\lokaltheoretisch\right\|_{\Xb\text{-}\infty} \notag \\
&\le \bsumme\  \|w_b\|_{\Xb\text{-}\infty}\  \underset{\tilde \Pb \in N_{b,\varepsilon_b}(\Pb)}{\sup} \    \left\| \lokaltheoretischschlange-\lokaltheoretisch\right\|_{\Xb\text{-}\infty} \notag \\
& \stackrel{\text{\small (\ref{dieungleichung})}}{\le} \bsumme\  \|w_b\|_{\Xb\text{-}\infty}\  \underset{\tilde \Pb \in N_{b,\varepsilon_b}(\Pb)}{\sup} \    \left\|k_b\right\|_{\Xb\text{-}\infty}\left\| \lokaltheoretischschlange-\lokaltheoretisch\right\|_{H_b} \notag \\
&\le \bsumme\  \|w_b\|_{\Xb\text{-}\infty}\  \underset{\tilde \Pb \in N_{b,\varepsilon_b}(\Pb)}{\sup} \    \left\|k_b\right\|_{\Xb\text{-}\infty} \varepsilon_b \frac 1 \lambda_b |L|_1 \left\|k_b\right\|_{\Xb\text{-}\infty} \|\tilde \Pb - \Pb\|_{\Xb\text{-TV}} \notag \\
&\le 2 \ |L|_1\ \bsumme\  \|w_b\|_{\Xb\text{-}\infty}\  \frac{\varepsilon_b }{\lambda_b}\ \left\|k_b\right\|_{\Xb\text{-}\infty}^2. \notag
\end{align}\notag

The second to last inequality is valid due to \fullciteA[Theorem~12, p.~316]{christmann2009consistency}, the last inequality due to the fact that the total variation norm of the difference of two probability measures can be bounded to 2 by the triangle inequality.
\end{proof}

\section{About shifted loss functions}\label{shiftanhang}

To explain why we like to use the shifted loss function instead of the loss function itself, we just have to look at the idea of SVMs. According to the  definition an SVM is a minimizer of a regularized risk. If the considered risk is $\infty$, there is nothing to be minimized. As shown in \fullciteA{christmann2009consistency}, there are two conditions for the finiteness of the risk of an SVM because for a Lipschitz-continuous loss function $L$ with Lipschitz-constant $|L|_1$ and under the assumption that $P$ can be split up into the marginal distribution $P^X$ and the regular conditional probability $P(y|x)$ we obtain
$$\left|\mathcal R_{L,P}(f)\right| = \left|\int\limits_{\X\times\Y} L(x,y,f(x)) \ dP(x,y)\right| = \left|\int\limits_{\X\times\Y} L(x,y,f(x))-\underbrace{L(x,y,y) }_{=0}\ dP(x,y)\right| \le$$
$$\le |L|_1  \int\limits_\X \int\limits_\Y |f(x)-y| \ dP(y|x) \ dP^X(x)\le |L|_1  \int\limits_\X  |f(x)| \  dP^X(x) + |L|_1  \int\limits_\X \int\limits_\Y |y| \ dP(y|x) \ dP^X(x).$$
We see that this expression is finite if $f\in L^1(P^X)$ \emph{and} $\int\limits_\X \int\limits_\Y |y| \ dP(y|x) \ dP^X(x) = \E_Q|Y|$ is finite. This might be an issue, if $\mathcal Y$ is unbounded, as, e.\,g. in general regression problems. Distributions without a finite first absolute moment, like, e.\,g. the Cauchy distribution, are therefore excluded if we use the loss function $L$ itself. Using the shifted version $\Ls$ results into
$$\left|\mathcal R_{\Ls,P}(f)\right| \le \int\limits_{\X\times\Y} \left|L(x,y,f(x))-L(x,y,0)\right| \ dP(x,y)\le |L|_1 \int\limits_\X |f(x)| \ dP^X(x),$$
which is finite if only $f\in L^1(P^X)$. Note that there is no moment condition for $Y$ anymore. Also note that this is \fullciteA[Proposition 3(ii), p.~314]{christmann2009consistency}.

Even if the considered loss function is not Lipschitz-continuous, shifting might be useful. To demonstrate this, let us have a look at the (only locally Lipschitz-continuous) least squares loss function which is defined by $L_{LS}(x,y,t) := (y-t)^2.$ The corresponding risk is, provided the decomposition of $P$ is allowed,  given by
\begin{align}
\mathcal R_{ {\X, L_{LS}},P }(f) & = \int\limits_\X\int\limits_\Y  \bigl(y-f(x)\bigr)^2\ dP(y|x)\ dP^X(x) = \notag \\
& = \int\limits_\X\int\limits_\Y  y^2-2yf(x) + f(x)^2\ dP(y|x)\ dP^X(x) = \notag \\
& = \int\limits_\X\ \ \ \left(\int\limits_\Y  y^2\ dP(y|x) -2f(x)\int\limits_\Y y  \ dP(y|x)\right) \ \ \ + f(x)^2\ dP^X(x) = \notag \\
& = \int\limits_\X\ \ \ \left(\ \int\limits_\Y  y^2\ dP(y|x) -2f(x)\int\limits_\Y y  \ dP(y|x)\ \right) \ \ \ dP^X(x) \ \ \ + \int\limits_\X f(x)^2\ dP^X(x) .\notag
\end{align}
We see that we obtain a well defined and finite risk if $f\in L^2(P^X)$ and, additionally, $\int\limits_\X \int\limits_\Y y^2 \ dP(y|x) \ dP^X(x) = \E_Q(Y^2)$ is finite. Shifting $L_{LS}$ leads to $\Ls_{LS}(x,y,t) = (y-t)^2-y^2$ with the corresponding risk
\begin{align}
\mathcal R_{\Ls_{LS},P}(f) & = \int\limits_\X\int\limits_\Y y^2-2yf(x)+f(x)^2-y^2\ dP(y|x)\ dP^X(x) = \notag \\
& = \int\limits_\X \left(-2f(x)\int\limits_\Y y \ dP(y|x) + f(x)^2\right) \ dP^X(x) \notag
\end{align}
which is finite and well defined if $f\in L^2(P^X)$ and $\E_Q(Y)$ is finite. Thus, using the shifted loss least squares loss function reduces the moment condition on $Y$ from a second order to first order condition.

The usage of shifted loss functions was, anticipatory and for M-estimators, introduced by \fullciteA{huber1967behavior}.

\end{appendix}

%\newpage{}
%\bibliographystyle{alpha}
\bibliographystyle{newapa}
\renewcommand{\bibname}{References}
\renewcommand{\refname}{References}
\addcontentsline{toc}{chapter}{\bibname}
%\bibliography{myBibliographie}

\end{document}